\documentclass{l4dc2024}

\usepackage{times}

\usepackage[utf8]{inputenc} %
\usepackage[T1]{fontenc}    %
\usepackage{nicefrac}       %
\usepackage{microtype}      %
\usepackage{xcolor}         %
\usepackage{enumitem}
\usepackage{comment}
\usepackage{hyperref}		%

\usepackage{amsmath, amssymb, amsfonts}
\newtheorem{assumption}{Assumption}

\usepackage{caption}
\usepackage{booktabs}
\usepackage{footnote}
\usepackage{chngpage}
\usepackage{placeins}
\usepackage{ragged2e}

\usepackage[all]{hypcap}	%
\usepackage[capitalize]{cleveref}
\Crefname{section}{Section}{Sections}
\Crefname{table}{Table}{Tables}

\crefformat{equation}{\textup{#2(#1)#3}}
\crefrangeformat{equation}{\textup{#3(#1)#4--#5(#2)#6}}
\crefmultiformat{equation}{\textup{#2(#1)#3}}{ and \textup{#2(#1)#3}}
{, \textup{#2(#1)#3}}{, and \textup{#2(#1)#3}}
\crefrangemultiformat{equation}{\textup{#3(#1)#4--#5(#2)#6}}%
{ and \textup{#3(#1)#4--#5(#2)#6}}{, \textup{#3(#1)#4--#5(#2)#6}}{, and \textup{#3(#1)#4--#5(#2)#6}}

\Crefformat{equation}{#2Equation~\textup{(#1)}#3}
\Crefrangeformat{equation}{Equations~\textup{#3(#1)#4--#5(#2)#6}}
\Crefmultiformat{equation}{Equations~\textup{#2(#1)#3}}{ and \textup{#2(#1)#3}}
{, \textup{#2(#1)#3}}{, and \textup{#2(#1)#3}}
\Crefrangemultiformat{equation}{Equations~\textup{#3(#1)#4--#5(#2)#6}}%
{ and \textup{#3(#1)#4--#5(#2)#6}}{, \textup{#3(#1)#4--#5(#2)#6}}{, and \textup{#3(#1)#4--#5(#2)#6}}

\crefdefaultlabelformat{#2\textup{#1}#3}

\usepackage[utf8]{inputenc}
\usepackage{amsfonts,amsmath,amssymb} %
\usepackage{braket}
\usepackage{mathtools,xparse}
\usepackage{bbm}
\def\eqref#1{equation~\ref{#1}}

\def\1{\bm{1}}

\DeclareMathAlphabet{\mathsfit}{\encodingdefault}{\sfdefault}{m}{sl}
\SetMathAlphabet{\mathsfit}{bold}{\encodingdefault}{\sfdefault}{bx}{n}

\def\gN{{\mathcal{N}}}

\def\gS{{\mathcal{S}}}

\def\sI{{\mathbb{I}}}

\def\sP{{\mathbb{P}}}

\def\sR{{\mathbb{R}}}

\newcommand{\E}{\mathbb{E}}

\DeclareMathOperator{\sgn}{sgn}

\DeclareMathOperator*{\argmax}{arg\,max}

\DeclarePairedDelimiter\norm{\lVert}{\rVert}
\DeclarePairedDelimiter\bignorm{\big\lVert}{\big\rVert}

\DeclarePairedDelimiterX{\inp}[2]{\big\langle}{\big\rangle}{#1, #2}

\newcommand{\hismoa}{h_{\text{smo1}, i}^{\gamma}}
\newcommand{\hismob}{h_{\text{smo2}, i}^{\gamma}}
\newcommand{\hismoc}{h_{\text{smo3}, i}^{\gamma}}
\newcommand{\hyalpha}{h_y^{\alpha}}
\newcommand{\hialpha}{h_i^{\alpha}}

\newcommand{\halpha}{h^{\alpha}}
\newcommand{\hbase}{h_\text{baseline}}
\DeclareMathOperator{\lip}{Lip}

\title[Mixing Classifiers to Alleviate the Accuracy-Robustness Trade-Off]{Mixing Classifiers to Alleviate the Accuracy-Robustness Trade-Off}

\coltauthor{\Name{Yatong Bai} \Email{\href{mailto:yatong_bai@berkeley.edu}{yatong\_bai@berkeley.edu}}\\
  \Name{Brendon G. Anderson} \Email{\href{mailto:bganderson@berkeley.edu}{bganderson@berkeley.edu}}\\
  \Name{Somayeh Sojoudi} \Email{\href{mailto:sojoudi@berkeley.edu}{sojoudi@berkeley.edu}}\\
  \addr University of California, Berkeley}

\begin{document}

\maketitle

\begin{abstract}%
	Deep neural classifiers have recently found tremendous success in data-driven control systems. However, existing models suffer from a trade-off between accuracy and adversarial robustness. This limitation must be overcome in the control of safety-critical systems that require both high performance and rigorous robustness guarantees. In this work, we develop classifiers that simultaneously inherit high robustness from robust models and high accuracy from standard models. Specifically, we propose a theoretically motivated formulation that mixes the output probabilities of a standard neural network and a robust neural network. Both base classifiers are pre-trained, and thus our method does not require additional training. Our numerical experiments verify that the mixed classifier noticeably improves the accuracy-robustness trade-off and identify the confidence property of the robust base classifier as the key leverage of this more benign trade-off. Our theoretical results prove that under mild assumptions, when the robustness of the robust base model is certifiable, no alteration or attack within a closed-form $\ell_p$ radius on an input can result in misclassification of the mixed classifier.
\end{abstract}


\section{Introduction} \label{sec:intro}

In recent years, high-performance machine learning models have been employed in various control settings, including reinforcement learning for dynamic systems with uncertainty \citep{Levine16, Sutton18} and autonomous driving \citep{bojarski2016end, Wu17}. However, models such as neural networks have been shown to be vulnerable to adversarial attacks, which are imperceptibly small input data alterations maliciously designed to cause failure \citep{Szegedy14, Nguyen15, Huang17, eykholt2018robust, liu2019perceptual}. This vulnerability makes such models unreliable for safety-critical control where guaranteeing robustness is necessary. In response, ``adversarial training (AT)'' \citep{Kurakin17, Goodfellow15, Bai22a, Bai22b, Zheng20, Zhang19} have been studied to alleviate the susceptibility. AT builds robust neural networks by training on adversarially attacked data.

A parallel line of work focuses on mathematically certified robustness \citep{Anderson20, Ma20, Anderson21a}. Among these methods, ``randomized smoothing (RS)'' is a particularly popular one that seeks to achieve certified robustness by processing intentionally corrupted data at inference time \citep{Cohen19c, Li19, Pfrommer22}, and has recently been applied to robustify reinforcement learning-based control strategies \citep{Kumar22, Wu22}. The recent work \citep{Anderson21b} has shown that ``locally biased smoothing,'' which robustifies the model locally based on the input test datum, outperforms the traditional RS with fixed smoothing noise. However, \citet{Anderson21b} only focus on binary classification problems, significantly limiting the applications. Moreover, \citet{Anderson21b} rely on the robustness of a $K$-nearest-neighbor ($K$-NN) classifier, which suffers from a lack of representation power when applied to harder problems and becomes a bottleneck.

While some works have shown that there exists a fundamental trade-off between accuracy and robustness \citep{Tsipas19, Zhang19}, recent research has argued that it should be possible to simultaneously achieve robustness and accuracy on benchmark datasets \citep{Yang20}. To this end, variants of AT that improve the accuracy-robustness trade-off have been proposed, including TRADES \citep{Zhang19}, Interpolated Adversarial Training \citep{Lamb19}, and many others \citep{Raghunathan20, Wang19b, Tramer18, Balaji19}. However, even with these improvements, degraded clean accuracy is often an inevitable price of achieving robustness. Moreover, standard non-robust models often achieve enormous performance gains by pre-training on larger datasets, whereas the effect of pre-training on robust classifiers is less understood and may be less prominent \citep{Chen20, Fan21}.

This work makes a theoretically disciplined step towards robustifying models without sacrificing clean accuracy. Specifically, we build upon locally biased smoothing and replace its underlying $K$-NN classifier with a robust neural network that can be obtained via various existing methods. We then modify how the standard base model (a highly accurate but possibly non-robust neural network) and the robust base model are ``mixed'' accordingly. The resulting formulation, to be introduced in \cref{sec:STD+ROB}, is a convex combination of the output probabilities from the two base classifiers. We prove that, when the robust network has a bounded Lipschitz constant or is built via RS, the mixed classifier also has a closed-form certified robust radius. More importantly, our method achieves an empirical robustness level close to that of the robust base model while approaching the standard base model's clean accuracy. This desirable behavior significantly improves the accuracy-robustness trade-off, especially for tasks where standard models noticeably outperform robust models on clean data.

Note that we do not make any assumptions about how the standard and robust base models are obtained (can be AT, RS, or others), nor do we assume the adversarial attack type and budget. Thus, our mixed classification scheme can take advantage of pre-training on large datasets via the standard base classifier and benefit from ever-improving robust training methods via the robust base classifier.

\section{Background and related works}

\subsection{Notations}

The $\ell_p$ norm is denoted by $\norm{\cdot}_p$, while $\norm{\cdot}_{p*}$ denotes its dual norm. The matrix $I_d$ denotes the identity matrix in $\sR^{d \times d}$. For a scalar $a$, $\sgn (a) \in \{ -1, 0, 1 \}$ denotes its sign. For a natural number $c$, the set $[c]$ is defined as $\{1, 2, \dots, c\}$. For an event $A$, the indicator function $\sI (A)$ evaluates to 1 if $A$ takes place and 0 otherwise. The notation $\sP_{X \sim \gS} [A (X)]$ denotes the probability for an event $A (X)$ to occur, where $X$ is a random variable drawn from the distribution $\gS$. The normal distribution on $\sR^d$ with mean $\overline{x}$ and covariance $\Sigma$ is written as $\gN (\overline{x}, \Sigma)$. We denote the cumulative distribution function of $\gN(0, 1)$ on $\sR$ by $\Phi$ and write its inverse function as $\Phi^{-1}$.

Consider a model $g: \sR^d \to \sR^c$, whose components are $g_i: \sR^d \to \sR,\ i \in [c]$, where $d$ is the dimension of the input and $c$ is the number of classes. In this paper, we assume that $g (\cdot)$ does not have the desired level of robustness, and refer to it as a ``standard model'', as opposed to a ``robust model'' which we denote as $h (\cdot)$.
We consider $\ell_p$ norm-bounded attacks on differentiable neural networks. A classifier $f: \sR^d \to [c]$, defined as $f(x) = \argmax_{i \in [c]} g_i (x)$, is considered robust against adversarial attacks at an input datum $x \in \sR^d$ if it assigns the same class to all perturbed inputs $x + \delta$ such that $\norm{\delta}_p \leq \epsilon$, where $\epsilon \geq 0$ is the attack radius.

\subsection{Related Adversarial Attacks and Defenses}

The fast gradient sign method (FGSM) and projected gradient descent (PGD) attacks based on differentiating the cross-entropy loss are highly effective and have been considered the most standard attacks for evaluating robust models \citep{Madry18, Goodfellow15}. To exploit the structures of the defense methods, adaptive attacks have also been introduced \citep{Tramer20}.

On the defense side, while AT \citep{Madry18} and TRADES \citep{Zhang19} have seen enormous success, such methods are often limited by a significantly larger amount of required training data \citep{Schmidt18} and a decrease in generalization capability. Initiatives that construct more effective training data via data augmentation \citep{Rebuffi21, Gowal21} and generative models \citep{Sehwag22} have successfully produced more robust models. Improved versions of AT \citep{Jia22, Shafahi19} have also been proposed.

Previous initiatives that aim to enhance the accuracy-robustness trade-off include using alternative attacks during training \citep{Pang22}, appending early-exit side branches to a single network \citep{Hu20}, and applying AT for regularization \citep{Zheng21}. Moreover, ensemble-based defenses, such as random ensemble \citep{Liu18} and diverse ensemble \citep{Pang19, Alam22}, have been proposed. In comparison, this work considers two separate classifiers and uses their synergy to improve the accuracy-robustness trade-off, achieving higher performances.

\subsection{Locally Biased Smoothing}

Randomized smoothing, popularized by \citep{Cohen19c}, achieves robustness at inference time by replacing $f (x) = \argmax_{i\in[c]} g_i(x)$ with a smoothed classifier \scalebox{.955}[1]{$\widetilde{f} (x) = \argmax_{i \in [c]} \E_{\xi \sim \gS} \left[ g_i (x + \xi) \right]$}, where $\gS$ is a smoothing distribution. A common choice for $\gS$ is a Gaussian distribution.

\citet{Anderson21b} have recently argued that data-invariant RS does not always achieve robustness. They have shown that in the binary classification setting, RS with an unbiased distribution is suboptimal, and an optimal smoothing procedure shifts the input point in the direction of its true class. Since the true class is generally unavailable, a ``direction oracle'' is used as a surrogate. This ``locally biased smoothing'' method is no longer randomized and outperforms traditional data-blind RS. The locally biased smoothed classifier, denoted $h^{\gamma} \colon \sR^d \to \sR$, is obtained via the deterministic calculation $h^{\gamma} (x) = g(x) + \gamma h(x) \norm{\nabla g(x)}_{p*}$, where $h(x) \in \{ -1, 1 \}$ is the direction oracle and $\gamma \geq 0$ is a trade-off parameter. The direction oracle should come from an inherently robust classifier (which is often less accurate). In \citep{Anderson21b}, this direction oracle is chosen to be a one-nearest-neighbor classifier.

\section{Using a Robust Neural Network as the Smoothing Oracle} \label{sec:STD+ROB}

Locally biased smoothing was designed for binary classification, restricting its practicality. Here, we first extend it to the multi-class setting by treating the output of each class, denoted as $h^\gamma_i (x)$, independently, giving rise to:
\vspace{-2mm}
\begin{equation} \label{eq:adap_sm_1}
    \hismoa (x) \coloneqq g_i (x) + \gamma h_i (x) \norm{\nabla g_i (x)}_{p*}, \;\;\; i \in [c].
\end{equation}

Note that if $\norm{\nabla g_i (x)}_{p*}$ is large for some class $i$, then $\hismoa (x)$ can be large for class $i$ even if both $g_i (x)$ and $h_i (x)$ are small, leading to incorrect predictions. To remove the effect of the gradient magnitude difference across the classes, we propose a normalized formulation as follows:
\begin{equation} \label{eq:adap_sm_2}
    \hismob (x) \coloneqq \frac{g_i (x) + \gamma h_i (x) \norm{\nabla g_i (x)}_{p*}}{1 + \gamma \norm{\nabla g_i (x)}_{p*}}, \;\;\; i \in [c].
\end{equation}

The parameter $\gamma$ adjusts between clean accuracy and robustness. It holds that $\hismob (x) \equiv g_i (x)$ when $\gamma = 0$, and $\hismob (x) \to h_i (x)$ when $\gamma \to \infty$ for all $x$ and all $i$.

With the mixing procedure generalized to the multi-class setting, we now discuss the choice of the smoothing oracle $h_i (\cdot)$. While $K$-NN classifiers are relatively robust and can be used as the oracle, their representation power is too weak. On the CIFAR-10 image classification task \citep{cifar10}, $K$-NN only achieves around $35\%$ accuracy on clean test data. In contrast, an adversarially trained ResNet can reach $50\%$ accuracy on attacked test data \citep{Madry18}. This lackluster performance of $K$-NN becomes a significant bottleneck in the accuracy-robustness trade-off of the mixed classifier. To this end, we replace the $K$-NN model with a robust neural network. The robustness of this network can be achieved via various methods, including AT, TRADES, and RS.

Further scrutinizing \cref{eq:adap_sm_2} leads to the question of whether $\norm{\nabla g_i (x)}_{p*}$ is the best choice for adjusting the mixture of $g (\cdot)$ and $h (\cdot)$. This gradient magnitude term is a result of \citet{Anderson21b}'s assumption that $h (x) \in \{-1, 1\}$. Here, we no longer have this assumption. Instead, we assume both $g (\cdot)$ and $h (\cdot)$ to be differentiable. Thus, we generalize the formulation to
\begin{gather} \label{eq:adap_sm_3}
    \hismoc (x) \coloneqq \frac{g_i(x) + \gamma R_i(x) h_i(x)}{1 + \gamma R_i(x)}, \;\;\; i\in[c],
\end{gather}
where $R_i (x)$ is an extra scalar term that can potentially depend on both $\nabla g_i (x)$ and $\nabla h_i (x)$ to determine the ``trustworthiness'' of the base classifiers. Here, we empirically compare four options for $R_i (x)$, namely, $1$, $\norm{\nabla g_i (x)}_{p*}$, $\norm{\nabla \max_j g_j (x)}_{p*}$, and $\frac{\norm{\nabla g_i (x)}_{p*}} {\norm{\nabla h_i (x)}_{p*}}$.

Another design question is whether $g (\cdot)$ and $h (\cdot)$ should be the pre-softmax logits or the post-softmax probabilities. Note that since most attack methods are designed based on logits, the output of the mixed classifier should be logits rather than probabilities to avoid gradient masking, an undesirable phenomenon that makes it hard to evaluate the robustness properly. Thus, we have the following two options that make the mixed model compatible with existing gradient-based attacks:
\begin{enumerate}
	\setlength\itemsep{0pt}
	\item Use the logits for both base classifiers, $g (\cdot)$ and $h (\cdot)$.
	\item Use the probabilities for both base classifiers, and then convert the mixed probabilities back to logits. The required ``inverse-softmax'' operator is simply the natural logarithm.
\end{enumerate}

\begin{figure}[t]
	\centering
	\begin{minipage}{.52\textwidth}
		\includegraphics[width=\textwidth]{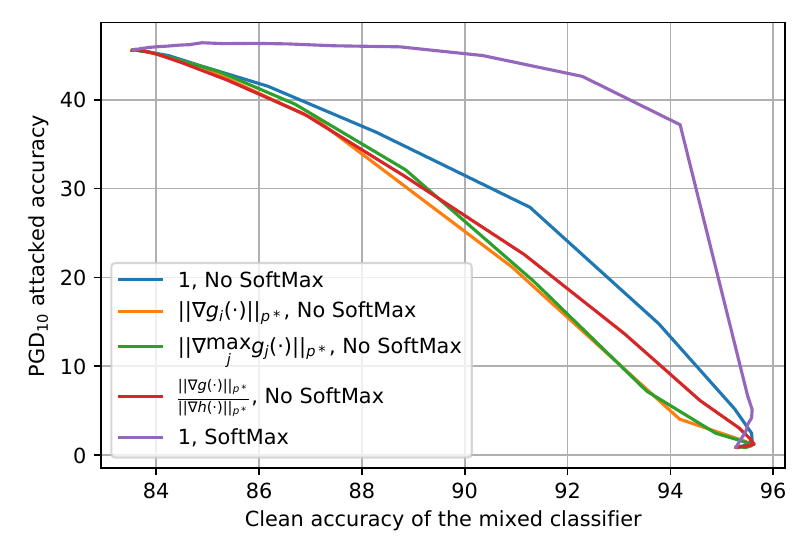}
	\end{minipage}
	\begin{minipage}{.37\textwidth}
		\begin{itemize}[leftmargin=4mm]
			\setlength\itemsep{.8em}
			\item \small ``No Softmax'' represents Option 1, i.e., use the logits for $g (\cdot)$ and $h (\cdot)$.
			\item \small ``Softmax'' represents Option 2, i.e., use the probabilities for $g (\cdot)$ and $h (\cdot)$.
			\item \small With the best formulation, high clean accuracy can be achieved with very little sacrifice on robustness.
		\end{itemize}
	\end{minipage}
	\vspace{-2mm}
	\caption{Comparing the ``attacked accuracy -- clean accuracy'' curves for various options for $R_i (x)$.}
	\label{fig:compare_R}
\end{figure}

\Cref{fig:compare_R} visualizes the accuracy-robustness trade-off achieved by mixing logits or probabilities with different $R_i (x)$ options. Here, the base classifiers are a pair of standard and adversarially trained ResNet-18s. This ``clean accuracy versus PGD$_{10}$-attacked accuracy'' plot concludes that $R_i (x) = 1$ gives the best accuracy-robustness trade-off, and $g (\cdot)$ and $h (\cdot)$ should be probabilities. \Cref{sec:more_compare_R} in the supplementary materials confirms this selection by repeating \Cref{fig:compare_R} with alternative model architectures, different robust base classifier training methods, and various attack budgets.

Our selection of $R_i (x) = 1$ differs from $R_i (x) = \norm{g_i (x)}_{p*}$ used in \citep{Anderson21b}. Intuitively, \citet{Anderson21b} used linear classifiers to motivate estimating the base models' trustworthiness with their gradient magnitudes. When the base classifiers are highly nonlinear neural networks as in our case, while a base classifier's local Lipschitzness correlates with its robustness, its gradient magnitude is not always a good local Lipschitzness estimator. Additionally, \Cref{sec:certified_radius_thms} offers theoretical intuitions for selecting mixing probabilities over mixing logits.

With these design choices implemented, the formulation \cref{eq:adap_sm_3} can be re-parameterized as
\begin{gather} \label{eq:adap_sm_4}
    \hialpha (x) \coloneqq \log \big( (1 - \alpha) g_i(x) + \alpha h_i(x) \big), \;\; i\in[c],
\end{gather}
where $\alpha = \frac{\gamma} {1 + \gamma} \in [0, 1]$. We take $\halpha (\cdot)$ in \cref{eq:adap_sm_4}, which is a convex combination of base classifier probabilities, as our proposed mixed classifier. Note that \cref{eq:adap_sm_4} calculates the mixed classifier logits, acting as a drop-in replacement for existing models which usually produce logits. Removing the logarithm recovers the output probabilities without changing the predicted class.

\subsection{Theoretical Certified Robust Radius} \label{sec:certified_radius_thms}

In this section, we derive certified robust radii for the mixed classifier $\halpha (\cdot)$ introduced in \cref{eq:adap_sm_4}, given in terms of the robustness properties of $h (\cdot)$ and the mixing parameter $\alpha$. The results ensure that despite being more sophisticated than a single model, $\halpha (\cdot)$ cannot be easily conquered, even if an adversary attempts to adapt its attack methods to its structure. Such guarantees are of paramount importance for reliable deployment in safety-critical control applications.

Noticing that the base model probabilities satisfy $0 \leq g_i (\cdot) \leq 1$ and $0 \leq h_i (\cdot) \leq 1$ for all $i$, we introduce the following generalized and tightened notion of certified robustness.

\begin{definition} \label{def:robust_with_margin}
	Consider an arbitrary input $x \in \sR^d$ and let $y = \argmax_i h_i (x)$, $\mu \in [0, 1]$, and $r \ge 0$. Then, $h(\cdot)$ is said to be \emph{certifiably robust at $x$ with margin $\mu$ and radius $r$} if $h_y (x + \delta) \geq h_i (x + \delta) + \mu$ for all $i \ne y$ and all $\delta \in \sR^d$ such that $\norm{\delta}_p \leq r$.
\end{definition}

\begin{lemma}
	\label{lem: certified_radius}
	Let $x \in \sR^d$ and $r \ge 0$. If it holds that $\alpha \in [\frac{1}{2}, 1]$ and $h (\cdot)$ is certifiably robust at $x$ with margin $\frac{1-\alpha} {\alpha}$ and radius $r$, then the mixed classifier $\halpha(\cdot)$ is robust in the sense that $\argmax_{i} \hialpha (x + \delta) = \argmax_{i} h_i (x)$ for all $\delta \in \sR^d$ such that $\norm{\delta}_p \leq r$.
\end{lemma}

\begin{proof}
Suppose that $h(\cdot)$ is certifiably robust at $x$ with margin $\frac{1-\alpha} {\alpha}$ and radius $r$. Since $\alpha \in [\frac{1}{2}, 1]$, it holds that $\frac{1-\alpha} {\alpha} \in [0, 1]$. Let $y = \argmax_i h_i(x)$. Consider an arbitrary $i \in [c] \setminus \{y\}$ and $\delta \in \sR^d$ such that $\norm{\delta}_p \le r$. Since $g_i (x+\delta) \in [0, 1]$, it holds that
\begin{align*}
	\exp \big( \hyalpha (x+\delta) \big) & - \exp \big( \hialpha (x+\delta) \big) \\
	& = (1-\alpha) (g_y (x+\delta) - g_i (x+\delta)) + \alpha (h_y (x+\delta) - h_i (x+\delta)) \\
	& \ge (1-\alpha) (0-1) + \alpha(h_y(x+\delta) - h_i(x+\delta)) \\
	& \ge (\alpha - 1) + \alpha \left( \tfrac{1-\alpha} {\alpha} \right) = 0.
\end{align*}
\vspace{-2mm}
Thus, it holds that $\hyalpha (x+\delta) \geq \hialpha (x+\delta)$ for all $i \neq y$, and thus $\argmax_i \hialpha (x+\delta) = y = \argmax_i h_i (x)$.
\end{proof}

Intuitively, Definition \ref{def:robust_with_margin} ensures that all points within a radius from a nominal point have the same prediction as the nominal point, with the difference between the top and runner-up probabilities no smaller than a threshold. For practical classifiers, the robust margin can be straightforwardly estimated by calculating the confidence gap between the predicted and the runner-up classes at an adversarial input obtained with strong attacks.

While most existing provably robust results consider the special case with zero margin, we will show that models built via common methods are also robust with non-zero margins. We specifically consider two types of popular robust classifiers: Lipschitz continuous models (\cref{thm: certified_radius}) and RS models (\cref{thm: randomized_smoothing}). Here, Lemma \ref{lem: certified_radius} builds the foundation for proving these two theorems, which amounts to showing that Lipschitz and RS models are robust with non-zero margins and thus the mixed classifiers built with them are robust.

Lemma \ref{lem: certified_radius} provides further justifications for using probabilities instead of logits in the mixing operation. Intuitively, it holds that $(1 - \alpha) g_i (\cdot)$ is bounded between $0$ and $1 - \alpha$, so as long as $\alpha$ is relatively large (specifically, at least $\frac{1}{2}$), the detrimental effect of $g (\cdot)$'s probabilities when subject to attack can be bounded and be overcome by $h (\cdot)$. Had we used the logits for $g_i (\cdot)$, since this quantity cannot be bounded, it would have been much harder to overcome the vulnerability of $g (\cdot)$.

Since we do not make assumptions on the Lipschitzness or robustness of $g (\cdot)$, Lemma \ref{lem: certified_radius} is tight. To understand this, we suppose that there exists some $i \in [c] \backslash \{ y \}$ and $\delta \neq 0$ such that $\norm{\delta}_p \leq r$ that make $h_y (x + \delta) - h_i (x + \delta) \coloneqq h_d$ smaller than $\frac{1-\alpha} {\alpha}$, indicating that $- \alpha h_d > \alpha-1$. Since the only information about $g (\cdot)$ is that $g_i (x + \delta) \in [0, 1]$ and thus the value $g_y (x + \delta) - g_i (x + \delta)$ can be any number in $[-1, 1]$, it is possible that $(1 - \alpha) \left( g_y (x + \delta) - g_i (x + \delta) \right)$ is smaller than $- \alpha h_d$. In this case, it holds that $\hyalpha (x + \delta) < \hialpha (x + \delta)$, and thus $\argmax_{i} \hialpha (x + \delta) \neq \argmax_{i} h_i (x)$.

\begin{definition}
	A function $f \colon \sR^d \to \sR$ is called \emph{$\ell_p$-Lipschitz continuous} if there exists $L \in (0, \infty)$ such that $|f(x')-f(x)| \le L\|x'-x\|_p$ for all $x', x \in \sR^d$. The \textbf{Lipschitz constant} of such $f$ is defined to be $\lip_p(f) \coloneqq \inf\{L\in(0,\infty) : |f(x')-f(x)| \le L\|x'-x\|_p ~ \text{for all $x', x \in \sR^d$}\}$.
\end{definition}

\begin{assumption}
	\label{ass: lipschitz}
	The classifier $h (\cdot)$ is robust in the sense that, for all $i \in \{1, 2, \dots, n\}$, $h_i (\cdot)$ is $\ell_p$-Lipschitz continuous with Lipschitz constant $\lip_p (h_i)$.
\end{assumption}

\begin{theorem}
	\label{thm: certified_radius}
	Suppose that Assumption \ref{ass: lipschitz} holds, and let $x \in \sR^d$ be arbitrary. Let $y = \argmax_i h_i(x)$. Then, if $\alpha \in [\frac{1}{2}, 1]$, it holds that $\argmax_{i} \hialpha (x + \delta) = y$ for all $\delta \in \sR^d$ such that
	\begin{equation} \label{eq:lip_cert_rad}
		\bignorm{\delta}_p \le r_p^\alpha (x) \coloneqq \min_{i \ne y} \frac{\alpha \left( h_y(x) - h_i(x) \right) + \alpha - 1} {\alpha \left( \lip_p(h_y) + \lip_p(h_i) \right)}.
	\end{equation}
\end{theorem}

\begin{proof}
Suppose that $\alpha \in [\frac{1}{2}, 1]$, and let $\delta \in \sR^d$ be such that $\norm{\delta}_p \le r_p^\alpha(x)$. Furthermore, let $i \in [c] \setminus \{y\}$. It holds that
\begin{align*}
	h_y (x+\delta) - h_i (x+\delta)
	&= h_y(x) - h_i(x) + h_y(x+\delta) - h_y(x) + h_i(x) - h_i(x+\delta) \\
	& \ge h_y(x) - h_i(x) - \lip_p(h_y) \norm{\delta}_p - \lip_p(h_i) \norm{\delta}_p \\
	& \ge h_y(x) - h_i(x) - \left( \lip_p(h_y) + \lip_p(h_i) \right) r_p^\alpha(x) \ge \tfrac{1 - \alpha} {\alpha}.
\end{align*}
Therefore, $h(\cdot)$ is certifiably robust at $x$ with margin $\frac{1-\alpha} {\alpha}$ and radius $r_p^\alpha(x)$. Hence, by Lemma \ref{lem: certified_radius}, the claim holds.
\end{proof}

We remark that the $\ell_p$ norm that \Cref{thm: certified_radius} certifies may be arbitrary (e.g., $\ell_1$, $\ell_2$, or $\ell_\infty$), so long as the Lipschitz constant of the robust network $h (\cdot)$ is computed with respect to the same norm.

Assumption \ref{ass: lipschitz} is not restrictive in practice. For example, Gaussian RS with smoothing variance $\sigma^2 I_d$ yields robust models with $\ell_2$-Lipschitz constant $\sqrt{ \nicefrac{2} {\pi \sigma^2} }$ \citep{Salman19}. Moreover, empirically robust methods such as AT and TRADES often train locally Lipschitz continuous models, even though there may not be closed-form theoretical guarantees.

Assumption \ref{ass: lipschitz} can be relaxed to the even less restrictive scenario of using local Lipschitz constants over a neighborhood (e.g., a norm ball) around a nominal input $x$ (i.e., how flat $h (\cdot)$ is near $x$) as a surrogate for the global Lipschitz constants. In this case, \cref{thm: certified_radius} holds for all $\delta$ within this neighborhood. Specifically, suppose that for an arbitrary input $x$ and an $\ell_p$ attack radius $\epsilon$, it holds that $h_y (x) - h_y (x + \delta) \le \epsilon \cdot \lip_p^x (h_y)$ and $h_i (x + \delta) - h_i (x) \le \epsilon \cdot \lip_p^x (h_i)$ for all $i \neq y$ and all perturbations $\delta$ such that $\norm{\delta}_p \leq \epsilon$. Furthermore, suppose that the robust radius $r_p^\alpha (x)$, as defined in \cref{eq:lip_cert_rad} but use the local Lipschitz constant $\lip_p^x$ as a surrogate to the global constant $\lip_p$, is not smaller than $\epsilon$. Then, if the robust base classifier $h (\cdot)$ is correct at the nominal point $x$, then the mixed classifier $\halpha (\cdot)$ is robust at $x$ within the radius $\epsilon$. The proof follows that of \cref{thm: certified_radius}.

The relaxed Lipschitzness defined above can be estimated for practical differentiable classifiers via an algorithm similar to the PGD attack \citep{Yang20}. \citet{Yang20} also showed that many existing empirically robust models, including those trained with AT or TRADES, are in fact locally Lipschitz. Note that \citet{Yang20} evaluated the local Lipschitz constants of the logits, whereas we analyze the probabilities, whose Lipschitz constants are much smaller. Therefore, \cref{thm: certified_radius} provides important insights into the empirical robustness of the mixed classifier.

An intuitive explanation of \Cref{thm: certified_radius} is that if $\alpha \to 1$, then $r_p^\alpha(x) \to \min_{i \ne y}\frac{h_y(x) - h_i(x)}{\lip_p(h_y) + \lip_p(h_i)}$, which is the standard Lipschitz-based robust radius of $h (\cdot)$ around $x$ (see \citep{Fazlyab19, Hein17} for further discussions on Lipschitz-based robustness). On the other hand, if $\alpha$ is too small in comparison to the relative confidence of $h (\cdot)$ and put an excess weight into the non-robust classifier $g (\cdot)$, namely, if there exists $i \ne y$ such that $\alpha \le \frac{1} {1 + h_y(x) - h_i(x)}$, then $r_p^\alpha (x) \le 0$, and in this case, we cannot provide non-trivial certified robustness for $\halpha (\cdot)$. If $h (\cdot)$ is $100\%$ confident in its prediction, then $h_y(x) - h_i(x) = 1$ for all $i \ne y$, and therefore this threshold value of $\alpha$ becomes $\frac{1}{2}$, leading to non-trivial certified radii for $\alpha > \frac{1}{2}$. However, once we put over $\frac{1}{2}$ of the weight into $g (\cdot)$, a nonzero radius around $x$ is no longer certifiable. Since no assumptions on the robustness of $g (\cdot)$ around $x$ have been made, this is intuitively the best one can expect.

We now move on to tightening the certified radius in the special case when $h (\cdot)$ is an RS classifier and our robust radii are defined in terms of the $\ell_2$ norm.

\begin{assumption}
	\label{ass: randomized smoothing}
	The classifier $h(\cdot)$ is a (Gaussian) randomized smoothing classifier, i.e., $h(x) = \E_{\xi \sim \gN (0, \sigma^2 I_d)} \left[ \overline{h} (x+\xi) \right]$ for all $x \in \sR^d$, where $\overline{h} \colon \sR^d \to [0, 1]^c$ is a neural model that is non-robust in general. Furthermore, for all $i \in [c]$, $\overline{h}_i (\cdot)$ is not 0 almost everywhere or 1 almost everywhere.
\end{assumption}

\begin{theorem}
	\label{thm: randomized_smoothing}
	Suppose that Assumption \ref{ass: randomized smoothing} holds, and let $x \in \sR^d$ be arbitrary. Let $y = \argmax_i h_i(x)$ and $y' = \argmax_{i \ne y} h_i(x)$. Then, if $\alpha \in [\frac{1}{2}, 1]$, it holds that $\argmax_i \hialpha (x+\delta) = y$ for all $\delta \in \sR^d$ such that
	\vspace{-1.3mm}
	\begin{align*}
		\norm{\delta}_2 &\le r_\sigma^\alpha(x)
		\coloneqq \frac{\sigma}{2} \Big( \Phi^{-1} \left( \alpha h_y(x)\right) - \Phi^{-1} \left( \alpha h_{y'} (x) + 1 - \alpha \right) \Big).
	\end{align*}
	\vspace{-7.6mm}
\end{theorem}

The proof of \Cref{thm: randomized_smoothing} is provided in \Cref{sec:rs_proof} in the supplementary materials.

To summarize our certified radii, \cref{thm: certified_radius} applies to very general Lipschitz continuous robust base classifiers $h(\cdot)$ and arbitrary $\ell_p$ norms, whereas \cref{thm: randomized_smoothing}, applying to the $\ell_2$ norm and RS base classifiers, strengthens the certified radius by exploiting the stronger Lipschitzness arising from the special structure and smoothness granted by Gaussian convolution operations. Theorems \ref{thm: certified_radius} and \ref{thm: randomized_smoothing} guarantee that our proposed robustification cannot be easily circumvented by adaptive attacks.

\section{Numerical Experiments}

\subsection{\texorpdfstring{$\alpha$}{alpha}'s Influence on Mixed Classifier Robustness} \label{sec:exp_CNN_fix}

We first use the CIFAR-10 dataset to evaluate the mixed classifier $\halpha (\cdot)$ with various values of $\alpha$. We use a ResNet18 model trained on unattacked images as the standard base model $g (\cdot)$ and use another ResNet18 trained on PGD$_{20}$ data as the robust base model $h (\cdot)$. We consider PGD$_{20}$ attacks that target $g (\cdot)$ and $h (\cdot)$ individually (abbreviated as STD and ROB attacks and can be regarded as transfer attacks), in addition to the adaptive PGD$_{20}$ attack generated using the end-to-end gradient of $\halpha (\cdot)$, denoted as the MIX attack.

The test accuracy of each mixed classifier is presented in \Cref{fig:STD+ROB}. As $\alpha$ increases, the clean accuracy of $\halpha (\cdot)$ converges from the clean accuracy of $g (\cdot)$ to the clean accuracy of $h (\cdot)$. In terms of attacked performance, when the attack targets $g (\cdot)$, the attacked accuracy increases with $\alpha$. When the attack targets $h (\cdot)$, the attacked accuracy decreases with $\alpha$, showing that the attack targeting $h (\cdot)$ becomes more benign when the mixed classifier emphasizes $g (\cdot)$. When the attack targets the mixed classifier $\halpha (\cdot)$, the attacked accuracy increases with $\alpha$.

When $\alpha$ is around $0.5$, the MIX-attacked accuracy of $\halpha (\cdot)$ quickly increases from near zero to more than $30\%$ (two-thirds of $h (\cdot)$'s attacked accuracy). This observation precisely matches the theoretical intuition from \Cref{thm: certified_radius}. Meanwhile, when $\alpha$ is greater than $0.5$, the clean accuracy gradually decreases at a much slower rate, leading to the alleviated accuracy-robustness trade-off.

\subsection{The Relationship between $\halpha (\cdot)$'s Robustness and $h (\cdot)$'s Confidence} \label{sec:conf_properties}

This difference in how clean and attacked accuracy change with $\alpha$ can be explained by the prediction confidence of the robust base classifier $h (\cdot)$. Specifically,  \cref{tab:confidence} confirms that $h (\cdot)$ makes confident correct predictions even when under attack (average robust margin is $0.768$). Moreover, $h (\cdot)$'s robust margin follows a long-tail distribution: the median robust margin is $0.933$, much larger than the $0.768$ mean. Thus, most attacked inputs correctly classified by $h (\cdot)$ are highly confident (i.e., robust with large margins). As Lemma \ref{lem: certified_radius} suggests, such a property is precisely what the mixed classifier relies on. Intuitively, once $\alpha$ becomes greater than $0.5$ and gives $h (\cdot)$ more authority over $g (\cdot)$, $h (\cdot)$ can use its confidence to correct $g (\cdot)$'s mistakes under attack.

\begin{figure}[!t]
	\begin{minipage}{.64\textwidth}
		\centering
    	\includegraphics[width=.86\textwidth, height=.5\textwidth, trim={4mm 4.8mm 2.5mm 4mm}, clip] {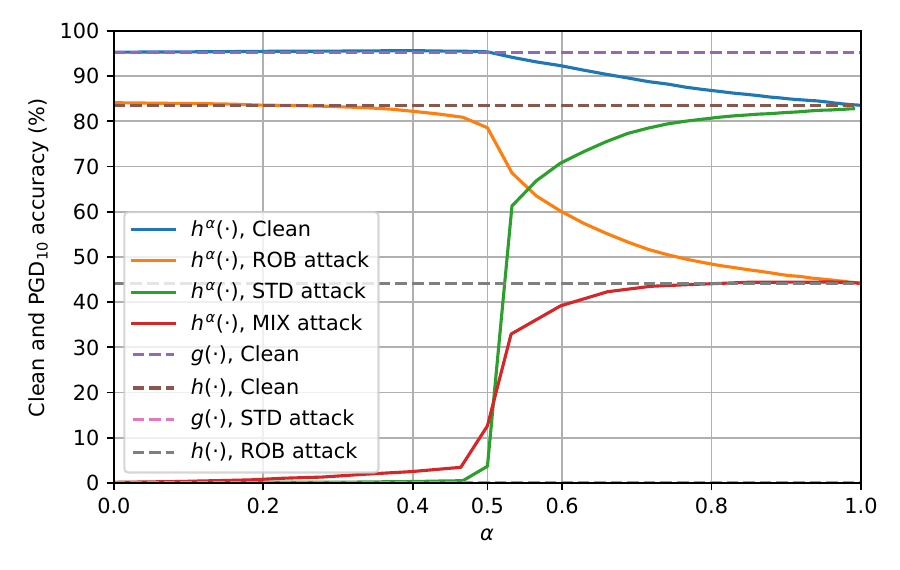}
    	\captionof{figure}{The accuracy of the mixed classifier $\halpha (\cdot)$ at various $\alpha$ values. ``STD attack'', ``ROB attack'', and ``MIX attack'' refer to the PGD$_{20}$ attack generated using the gradient of $g (\cdot)$, $h (\cdot)$, and $\halpha (\cdot)$ respectively, with $\epsilon$ set to $\frac{8}{255}$.}
    	\label{fig:STD+ROB}
	\end{minipage}
	\hfill
	\begin{minipage}{.35\textwidth}
		\centering
		\captionof{table}{Average gap between the probabilities of the predicted class and the runner-up class.}
		\label{tab:confidence}
		\begin{small}
		\begin{tabular}{l|c|c}
			\toprule
			& \multicolumn{2}{c}{Clean Instances} \\
			& Correct & Incorrect \\
			\midrule
			$g (\cdot)$ & 0.982 & 0.698 \\
			$h (\cdot)$ & 0.854 & 0.434 \\
			\bottomrule
		\end{tabular}

		$\ \ $\begin{tabular}{l|c|c}
			\toprule
			& \multicolumn{2}{c}{PGD$_{20}$ Instances} \\
			& Correct & Incorrect \\
			\midrule
			$g (\cdot)$ & 0.602 & 0.998 \\
			$h (\cdot)$ & 0.768 & 0.635 \\
			\bottomrule
		\end{tabular}
		\end{small}
		\vspace{2.5mm}
	\end{minipage}
\end{figure}

On the other hand, $h (\cdot)$ is unconfident when producing incorrect predictions on clean data, with the top two classes' output probabilities separated by merely $0.434$. This probability gap again forms a long-tail distribution (the median is $0.378$ which is less than the mean), confirming that $h (\cdot)$ rarely makes confident incorrect predictions. Now, consider clean data that $g (\cdot)$ correctly classifies and $h (\cdot)$ mispredicts. Recall that we assume $g (\cdot)$ to be more accurate but less robust, so this scenario should be common. Since $g (\cdot)$ is confident (average top two classes probability gap is $0.982$) and $h (\cdot)$ is usually unconfident, even when $\alpha > 0.5$ and $g (\cdot)$ has less authority than $h (\cdot)$ in the mixture, $g (\cdot)$ can still correct some of the mistakes from $h (\cdot)$.

In summary, $h (\cdot)$ is confident when making correct predictions on attacked data while being unconfident when misclassifying clean data, and such a confidence property is the key source of the mixed classifier's improved accuracy-robustness trade-off. Additional analyses in \Cref{sec:more_compare_R} with alternative base models imply that multiple existing robust classifiers share this benign confidence property and thus help the mixed classifier improve the trade-off.

\subsection{Visualization of the Certified Robust Radii}

Next, we visualize the certified robust radii presented in \cref{thm: certified_radius} and \cref{thm: randomized_smoothing}. Since a (Gaussian) RS model with smoothing covariance matrix $\sigma^2 I_d$ has an $\ell_2$-Lipschitz constant $\sqrt{\nicefrac{2}{\pi\sigma^2}}$, such a model can be used to simultaneously visualize both theorems, with \cref{thm: randomized_smoothing} giving tighter certificates of robustness. Note that RS models with a larger smoothing variance certify larger radii but achieve lower clean accuracy, and vice versa. Here, we consider the CIFAR-10 dataset and select $g (\cdot)$ to be a ConvNeXT-T model with a clean accuracy of $97.25 \%$, and use the RS models presented in \citep{Zhang19} as $h (\cdot)$. For a fair comparison, we select an $\alpha$ value such that the clean accuracy of the constructed mixed classifier $\halpha (\cdot)$ matches that of another RS model $\hbase (\cdot)$ with a smaller smoothing variance. The expectation term in the RS formulation is approximated with the empirical mean of $10000$ random perturbations drawn from $\gN(0, \sigma^2 I_d)$, and the certified radii of $\hbase(\cdot)$ are calculated using Theorems \ref{thm: certified_radius} and $\ref{thm: randomized_smoothing}$ by setting $\alpha$ to $1$. \Cref{fig: certified_radii} displays the calculated certified accuracy of $\halpha (\cdot)$ and $\hbase (\cdot)$ at various attack radii. The ordinate ``Accuracy'' at a given abscissa ``$\ell_2$ radius'' reflects the percentage of the test data for which the considered model gives a correct prediction as well as a certified radius at least as large as the $\ell_2$ radius under consideration.

\begin{figure*}[t]
	\centering
	\begin{subfigure}[\parbox{0.41\textwidth}{
            $\hbase(\cdot)$: RS with $\sigma=0.5$. \hspace{11.3mm} $ $ \\
            \hspace{.3mm} $\halpha(\cdot)$: $\alpha=0.76$; $h(\cdot)$ is RS with $\sigma=1$.
        }][t]{
        \captionsetup{justification=centering}
		\centering
		\includegraphics[height=.3075\textwidth]{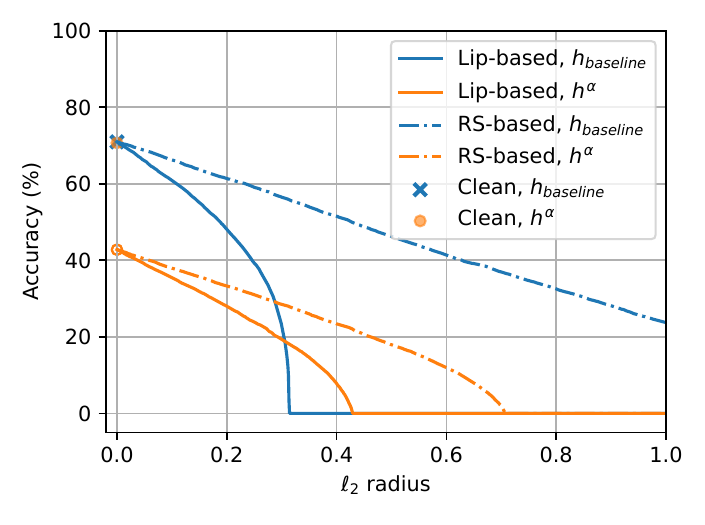}
		\vspace{-6.1mm}}
	\end{subfigure}
	\hfill
	\begin{subfigure}[\parbox{0.46\textwidth}{
			$\hbase(\cdot)$: RS with $\sigma=0.25$. \\
			Consider two mixed classifier examples: \\
			$\halpha_a (\cdot)$: $\alpha=0.76$; $h_a(\cdot)$ is RS with $\sigma=0.5$; \\
			$\hspace{.2mm} \halpha_b (\cdot)$: $\alpha=0.67$; $h_b(\cdot)$ is RS with $\sigma=1.0$.
		}][t]{
		\centering
		\includegraphics[height=.3075\textwidth]{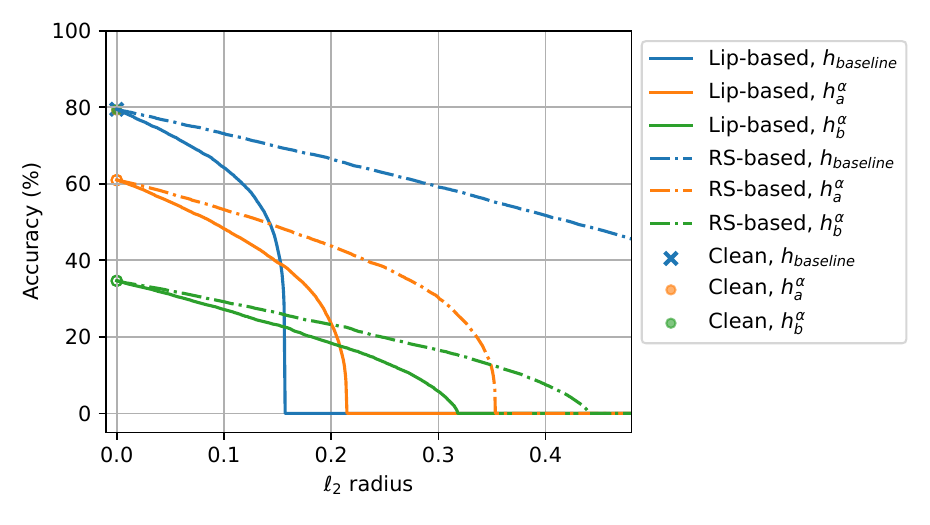}
		\vspace{-2.2mm}}
	\end{subfigure}
	\caption{Comparing the certified accuracy-robustness trade-off of RS models and our mixed classifier using both Lipschitz-based (Lip-based) certificates and RS-based certificates (Theorems \ref{thm: certified_radius} and \ref{thm: randomized_smoothing}, respectively). The clean accuracy is the same between $\hbase (\cdot)$ and $\halpha (\cdot)$ in each subfigure, and the empty circles represent discontinuity in the certified accuracy at radius $0$.}
	\label{fig: certified_radii}
\end{figure*}

In both subplots of \Cref{fig: certified_radii}, the certified robustness curves of $\halpha (\cdot)$ do not connect to the clean accuracy when $\alpha \to 0$. This is because Theorems \ref{thm: certified_radius} and \ref{thm: randomized_smoothing} both consider robustness with respect to $h (\cdot)$ and do not certify test inputs at which $h (\cdot)$ makes incorrect predictions, even though $\halpha (\cdot)$ may correctly predict some of these points. This is reasonable because we do not assume any robustness or Lipschitzness of $g (\cdot)$, and $g (\cdot)$ is allowed to be arbitrarily incorrect whenever the radius is non-zero.

The Lipschitz-based bound of \cref{thm: certified_radius} allows us to visualize the performance of the mixed classifier $\halpha (\cdot)$ when $h (\cdot)$ is an $\ell_2$-Lipschitz model. In this case, the curves associated with $\halpha (\cdot)$ and $\hbase (\cdot)$ intersect, with $\halpha (\cdot)$ achieving higher certified accuracy at larger radii and $\hbase (\cdot)$ certifying more points at smaller radii. Adjusting $\alpha$ and the Lipschitz constant of $h (\cdot)$ can change the location of this intersection while maintaining the clean accuracy. Thus, the mixed classifier allows for optimizing the certified accuracy at a particular radius without sacrificing clean accuracy.

The RS-based bound from \cref{thm: randomized_smoothing} captures the behavior of the mixed classifier $\halpha (\cdot)$ when $h (\cdot)$ is an RS model. For both $\halpha (\cdot)$ and $\hbase (\cdot)$, the RS-based bounds certify larger radii than the corresponding Lipschitz-based bounds. Nonetheless, $\hbase (\cdot)$ can certify more points with the RS-based guarantee. Intuitively, this phenomenon suggests that RS models can yield correct but low-confidence predictions when under large-radius attack, and thus may not be best-suited for our mixing operation, which relies on robustness with non-zero margins. Meanwhile, Lipschitz models, a more general and common class of models, exploit the mixing operation more effectively. Moreover, as shown in \Cref{fig:STD+ROB} and \Cref{tab:confidence}, empirically robust models often yield high-confidence correct predictions when under attack, making them more suitable to be used as $\halpha (\cdot)$'s robust base classifier.

\section{Conclusions}

This work proposes to mix the predicted probabilities of an accurate classifier and a robust classifier to mitigate the accuracy-robustness trade-off. These two base classifiers can be pre-trained, and the resulting mixed classifier requires no additional training. Theoretical results certify that the mixed classifier inherits the robustness of the robust base model under realistic assumptions. Empirical evaluations show that our method approaches the high accuracy of the latest standard models while retaining the robustness of modern robust classification methods. Hence, this work provides a foundation for future research to focus on either accuracy or robustness without sacrificing the other, providing additional incentives for deploying robust models in safety-critical control.

\newpage
\acks{This work was supported by grants from ONR, NSF, and C3 AI.}

\bibliography{Papers.bib}

\newpage
\appendix

\section{Additional Empirical Support for \texorpdfstring{$R_i(x) = 1$}{R\_i(x)=1}} \label{sec:more_compare_R}

\begin{figure*}[!t]
	\centering
	\begin{subfigure}[\footnotesize{ConvNeXT-T and TRADES WRN-34 under $\ell_\infty$ PGD attack.}][t]{
		\parbox{0.52\textwidth}{
			\centering
			\includegraphics[width=.45\textwidth]{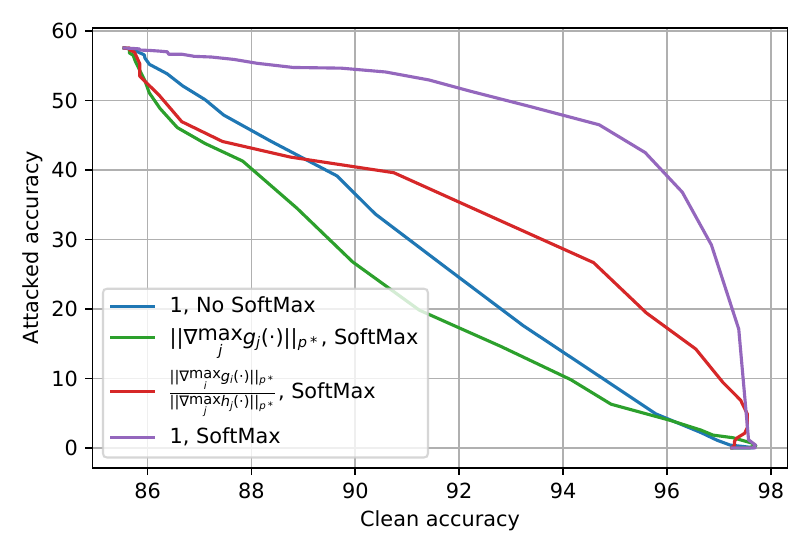}
			\label{fig:compare_R_2a}
		}}
	\end{subfigure}
	\hfill
	\begin{subfigure}[\footnotesize{Standard and AT ResNet18s under $\ell_2$ PGD attack.}][t]{
		\parbox{0.45\textwidth}{
			\centering
			\includegraphics[width=.45\textwidth]{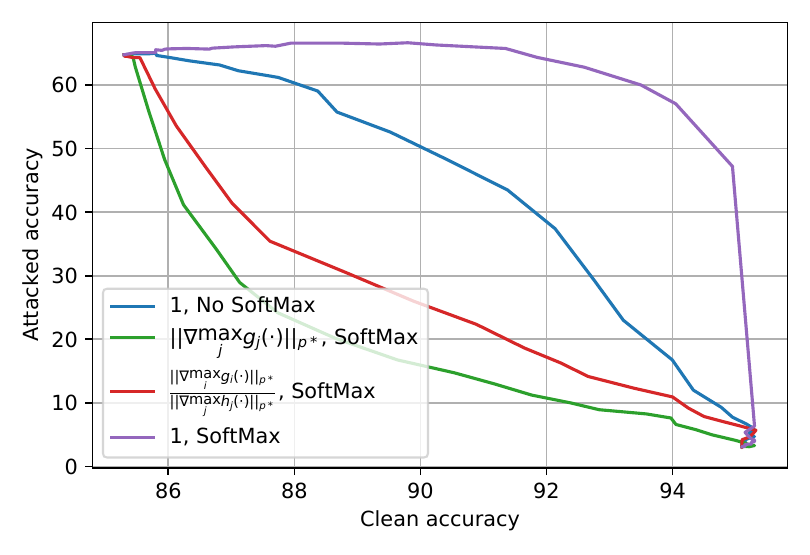}
			\label{fig:compare_R_2b}
			\vspace{-4.9mm}
		}}
	\end{subfigure}
	\vspace{-2mm}
	\caption{Comparing the options for $R_i (x)$ with alternative selections of base classifiers.}
	\label{fig:compare_R_2}
\end{figure*}

\begin{table*}[!t]
	\centering
	\caption{Experiment settings for comparing the choices of $R_i (x)$.}
	\begin{small}
	\begin{tabular}{l|c|c|c}
		\toprule
		& Attack Budget; PGD Steps & $g (\cdot)$ Architecture & $h (\cdot)$ Architecture \\
		\midrule
		\Cref{fig:compare_R}    & $\ell_\infty$, $\ \epsilon = \frac{8}{255}$, $ $10 Steps & Standard ResNet18   & $\ell_\infty$-adversarially-trained ResNet18 \\
		\Cref{fig:compare_R_2a} & $\ell_\infty$, $\ \epsilon = \frac{8}{255}$, $ $20 Steps & Standard ConvNeXT-T & TRADES WideResNet-34 \\
		\Cref{fig:compare_R_2b} & $\ell_2$, $\;\; \epsilon = 0.5$, $ $ 20 Steps		  		  & Standard ResNet18   & $\ell_2$-adversarially-trained ResNet18 \\
		\bottomrule
	\end{tabular}
	\end{small}
	\label{tab:compare_R_settings}
\end{table*}

Finally, we use additional empirical evidence (Figures \ref{fig:compare_R_2a} and \ref{fig:compare_R_2b}) to show that $R_i (x) = 1$ is the appropriate choice for the mixed classifier and that the probabilities should be used for the mixture. While most experiments in this paper are based on the popular ResNet architecture, our method does not depend on any ResNet properties. Therefore, for the experiment in \Cref{fig:compare_R_2a}, we select a more modern ConvNeXT-T model \citep{Liu22} pre-trained on ImageNet-1k as an alternative architecture for $g (\cdot)$. We also use a robust model trained via TRADES in place of an adversarially-trained network for $h (\cdot)$ for the interest of diversity. Additionally, although most of our experiments are based on $\ell_\infty$ attacks, the proposed method applies to all $\ell_p$ attack budgets. In \Cref{fig:compare_R_2b}, we provide an example that considers the $\ell_2$ attack. The experiment settings are summarized in \Cref{tab:compare_R_settings}.

Figures \ref{fig:compare_R_2a} and \ref{fig:compare_R_2b} confirm that setting $R_i (x)$ to the constant $1$ achieves the best trade-off curve between clean and attacked accuracy, and that mixing the probabilities outperforms mixing the logits. This result aligns with the conclusions of \Cref{fig:compare_R} and our theoretical analyses.

For all three cases listed in \cref{tab:compare_R_settings}, the mixed classifier reduces the error rate of $h (\cdot)$ on clean data by half while maintaining $80\%$ of $h (\cdot)$'s attacked accuracy. This observation suggests that the mixed classifier noticeably alleviates the accuracy-robustness trade-off. Additionally, our method is especially suitable for applications where the clean accuracy gap between $g (\cdot)$ and $h (\cdot)$ is large. On easier datasets such as MNIST and CIFAR-10, this gap has been greatly reduced by the latest advancements in constructing robust classifiers. However, on harder tasks such as CIFAR-100 and ImageNet-1k, this gap is still large, even for state-of-the-art methods. For these applications, standard classifiers often benefit much more from pre-training on larger datasets than robust models.

\section{Proof of \Cref{thm: randomized_smoothing}} \label{sec:rs_proof}

\setcounter{theorem}{4}
\begin{theorem}[Restated]
	Suppose that Assumption \ref{ass: randomized smoothing} holds, and let $x \in \sR^d$ be arbitrary. Let $y = \argmax_i h_i(x)$ and $y' = \argmax_{i \ne y} h_i(x)$. Then, if $\alpha \in [\frac{1}{2}, 1]$, it holds that $\argmax_i \hialpha (x+\delta) = y$ for all $\delta \in \sR^d$ such that
	\vspace{-1mm}
	\begin{align*}
		\norm{\delta}_2 &\le r_\sigma^\alpha(x)
		\coloneqq \frac{\sigma}{2} \Big( \Phi^{-1} \left( \alpha h_y(x)\right) - \Phi^{-1} \left( \alpha h_{y'} (x) + 1 - \alpha \right) \Big).
	\end{align*}
\end{theorem}

\begin{proof}
	First, note that since every $\overline{h}_i (\cdot)$ is not 0 almost everywhere or 1 almost everywhere, it holds that $h_i (x) \in (0, 1)$ for all $i$ and all $x$. Now, suppose that $\alpha \in [\frac{1}{2}, 1]$, and let $\delta\in \sR^d$ be such that $\norm{\delta}_2 \le r_\sigma^\alpha(x)$. Let $\mu_\alpha \coloneqq \frac{1-\alpha}{\alpha}$. Define the function $\tilde{h} \colon \sR^d \to \sR^c$ by
	\begin{equation*}
		\tilde{h}_y (x) = \frac{\overline{h}_y (x)} {1 + \mu_\alpha}, \quad
		\tilde{h}_i (x) = \frac{\overline{h}_i (x) + \mu_\alpha} {1 + \mu_\alpha} \text{ for all $i \ne y$}.
	\end{equation*}

	Furthermore, define $\hat{h} \colon \sR^d \to \sR^c$ by $\hat{h} (x) = \E_{\xi \sim \gN (0, \sigma^2 I_d)} \left[ \tilde{h} (x+\xi) \right]$.
	\vspace{1mm}

	Then, since $\tilde{h}_y(x) = \frac{\overline{h}_y(x)} {1+\mu_\alpha} \in (0, \frac{1}{1+\mu_\alpha}) \subseteq (0, 1)$ 
	and $\tilde{h}_i(x) = \frac{\overline{h}_i(x)+\mu_\alpha} {1+\mu_\alpha} \in (\frac{\mu_\alpha} {1+\mu_\alpha}, 1) \subseteq (0, 1)$ for all $i \neq y$, it must be the case that $0 < \tilde{h}_i (x) < 1$ for all $i$ and all $x$, and hence, for all $i$, the function $x \mapsto \Phi^{-1} \big( \hat{h}_i(x) \big)$ is $\ell_2$-Lipschitz continuous with Lipschitz constant $\frac{1}{\sigma}$ (see \cite[Lemma~1]{Levine19}, or Lemma 2 in \citep{Salman19} and the discussion thereafter). Therefore,
	\begin{equation}
		\left| \Phi^{-1} \big( \hat{h}_i(x+\delta) \big) - \Phi^{-1} \big( \hat{h}_i(x) \big) \right| \le \frac{\norm{\delta}_2} {\sigma} \le \frac{r_\sigma^\alpha(x)}{\sigma} \label{eq: lipschitz_inequality}
	\end{equation}
	for all $i$. Applying \cref{eq: lipschitz_inequality} for $i=y$ yields that
	\begin{equation} \label{eq:lip_y}
		\Phi^{-1} \big( \hat{h}_y (x+\delta) \big) \ge \Phi^{-1} \big( \hat{h}_y(x) \big) - \frac{r_\sigma^\alpha (x)} {\sigma}.
	\end{equation}
	Since $\Phi^{-1}$ monotonically increases and $\hat{h}_i(x) \le \hat{h}_{y'} (x)$ for all $i \ne y$, applying \cref{eq: lipschitz_inequality} to $i \ne y$ gives
	\begin{align} \label{eq:lip_not_y}
		\Phi^{-1} \big( \hat{h}_i (x+\delta) \big) & \le \Phi^{-1} \big( \hat{h}_i(x) \big) + \frac{r_\sigma^\alpha(x)} {\sigma}
		\le \Phi^{-1} \big( \hat{h}_{y'} (x) \big) + \frac{r_\sigma^\alpha(x)} {\sigma}.
	\end{align}
	Subtracting \cref{eq:lip_not_y} from \cref{eq:lip_y} gives that
	\begin{align*}
		\Phi^{-1} \big( \hat{h}_y (x+\delta) \big) & - \Phi^{-1} \big( \hat{h}_i (x+\delta) \big)
		\ge \Phi^{-1} \big(\hat{h}_y (x) \big) - \Phi^{-1} \big(\hat{h}_{y'} (x) \big) - \frac{2 r_\sigma^\alpha (x)} {\sigma}
	\end{align*}
	for all $i \ne y$. By the definitions of $\mu_\alpha$, $r_\sigma^\alpha (x)$, and $\hat{h} (x)$, the right-hand side of this inequality equals zero. Since $\Phi$ monotonically increases, we find that $\hat{h}_y(x+\delta) \ge \hat{h}_i(x+\delta)$ for all $i \ne y$. Thus,
	\begin{align*}
		\frac{h_y(x+\delta)} {1 + \mu_\alpha} & = \E_{\xi \sim \gN (0, \sigma^2 I_d)} \left[ \frac{\overline{h}_y (x+\delta+\xi)} {1+\mu_\alpha} \right] = \hat{h}_y(x+\delta) \\
		& \ge \hat{h}_i (x+\delta) = \E_{\xi \sim \gN (0, \sigma^2 I_d)} \left[ \frac{\overline{h}_i (x+\delta+\xi) + \mu_\alpha} {1+\mu_\alpha}\right]
		= \frac{h_i(x+\delta) + \mu_\alpha} {1+\mu_\alpha}.
	\end{align*}
	Hence, $h_y (x+\delta) \ge h_i (x+\delta) + \mu_\alpha$ for all $i \ne y$, so $h (\cdot)$ is certifiably robust at $x$ with margin $\mu_\alpha = \frac{1-\alpha} {\alpha}$ and radius $r_\sigma^\alpha (x)$. Therefore, by Lemma \ref{lem: certified_radius}, it holds that $\argmax_i \hialpha (x+\delta) = y$ for all $\delta \in \sR^d$ such that 
	$\norm{\delta}_2 \le r_\sigma^\alpha(x)$, which concludes the proof.
\end{proof}

\end{document}